\documentclass{article}

 \usepackage[preprint]{neurips_2026}

\usepackage[utf8]{inputenc} % allow utf-8 input
\usepackage[T1]{fontenc}    % use 8-bit T1 fonts
\usepackage{hyperref}       % hyperlinks
\usepackage{url}            % simple URL typesetting
\usepackage{booktabs}       % professional-quality tables
\usepackage{amsfonts}       % blackboard math symbols
\usepackage{nicefrac}       % compact symbols for 1/2, etc.
\usepackage{microtype}      % microtypography
\usepackage{xcolor}         % colors
\usepackage{microtype}
\usepackage{amsmath,amssymb,amsthm,mathtools}
\usepackage{booktabs}
\usepackage{hyperref}
\usepackage{algorithm}
\usepackage{algorithmic}
\usepackage{hyperref}
\usepackage{url}
\usepackage{xcolor}
\usepackage{graphicx}
\usepackage{enumitem}
\usepackage{multicol}
\usepackage{multirow}
\usepackage{wrapfig}

\usepackage{amsmath,amsfonts,bm}

\def\eqref#1{equation~\ref{#1}}
\def\1{\bm{1}}

\DeclareMathAlphabet{\mathsfit}{\encodingdefault}{\sfdefault}{m}{sl}
\SetMathAlphabet{\mathsfit}{bold}{\encodingdefault}{\sfdefault}{bx}{n}

\newcommand{\E}{\mathbb{E}}

\newcommand{\softmax}{\mathrm{softmax}}

\newcommand{\Var}{\mathrm{Var}}

\newtheorem{theorem}{Theorem}
\newtheorem{lemma}{Lemma}
\newtheorem{corollary}{Corollary}

\renewcommand{\E}{\mathbb{E}}
\renewcommand{\Var}{\mathrm{Var}}
\newcommand{\ESS}{\mathrm{ESS}}
\renewcommand{\softmax}{\mathrm{softmax}}
\newcommand{\vocab}{\mathcal{V}}

\newcommand{\EOS}{\text{EOS}}

\usepackage{xcolor}
\usepackage[table]{xcolor}

\definecolor{mygreen}{HTML}{8FCE00}

\title{Power-SMC: Low-Latency Sequence-Level Power Sampling for Training-Free LLM Reasoning}

\author{Seyedarmin Azizi$^u$, Erfan Baghaei Potraghloo$^u$, Minoo Ahmadi$^u$, Souvik Kundu$^i$, and \\ \textbf{Massoud Pedram}$^u$ \vspace{3mm}\\
$^u$University of Southern California,
Los Angeles, USA \\
$^i$Intel AI, USA\\
\texttt{\{seyedarm,baghaeip,minooahm,pedram\}@usc.edu},  \texttt{souvikk.kundu@intel.com}\\
}

\begin{document}

\maketitle

\maketitle
\begin{abstract}
Reasoning ability in large language models is often attributed to \emph{distribution sharpening}: concentrating output probability on high-likelihood sequences.
Recent works show that this sharpening effect can be obtained at inference time, without modifying model parameters, and can elicit strong reasoning performance.
A natural formalization is the \emph{sequence-level power distribution}, which is proportional to the model's probability raised to an exponent $\alpha>1$.
Prior work leveraged Metropolis--Hastings (MH) sampling to draw samples from this distribution and achieves strong results, however, at order-of-magnitude inference slowdowns.
We introduce \textbf{Power-SMC}, a \textit{`training-free'} sampling method that targets the same power distribution yielding close to standard decoding latency.
Power-SMC maintains multiple candidate sequences in parallel.
Each candidate sequence is assigned a score, namely the \emph{importance weight}, that measures how well it matches the power distribution.
It then periodically prunes low-scoring candidate sequences in favor of high-scoring ones.
We further provide a theoretical justification for the design choices in Power-SMC.
Among all next-token sampling strategies that do not rely on future tokens, we prove that sampling temperature $\tau{=}1/\alpha$ uniquely eliminates per-step weight variance.
Finally, we characterize the remaining source of weight instability and introduce a gradual sharpening schedule to reduce weight collapse, while targeting the same power distribution.
Extensive evaluations on MATH500, GSM8K, GPQA, and HumanEval show that Power-SMC matches or exceeds MH sampling in accuracy, preserves output diversity unlike RL-finetuned models, while  \textbf{accelerating inference speed by up to} {$\mathbf{17.6}\times$}. The code is available at \href{https://github.com/ArminAzizi98/Power-SMC}{https://github.com/ArminAzizi98/Power-SMC}.
\end{abstract}

\section{Introduction}
\label{sec:intro}
% ============================================================
Large language models (LLMs) have achieved strong reasoning capabilities through post-training methods such as reinforcement learning (RL).
Recent work argues that the effect of RL post-training can be understood as \emph{distribution sharpening}~\citep{karan2026reasoning, yue2025does}.
After RL post-training, the model concentrates its outputs on high-likelihood sequences that already had nonzero probability under the pretrained distribution.
This raises a natural question: can we achieve the same sharpening effect at inference time, without modifying model parameters at all?
One way to formalize this is through the \emph{sequence-level power distribution}.
Given a prompt~$x$, let $p_\theta$ be a pretrained autoregressive LLM that assigns probability $p_\theta(y\mid x)$ to each output sequence $y$.
The power distribution is defined as
\vspace{-2mm}
\begin{equation}
\pi_\alpha(y\mid x)\;=\;\frac{p_\theta(y\mid x)^\alpha}{Z_\alpha(x)},
\qquad
Z_\alpha(x)=\sum_{y} p_\theta(y\mid x)^\alpha.
\label{eq:power}
\end{equation}
Intuitively, choosing $\alpha>1$ concentrates probability on high-likelihood sequences without changing the model parameters (Figure~\ref{fig:concept}).
\begin{figure}[t]
\centering
\includegraphics[width=\linewidth]{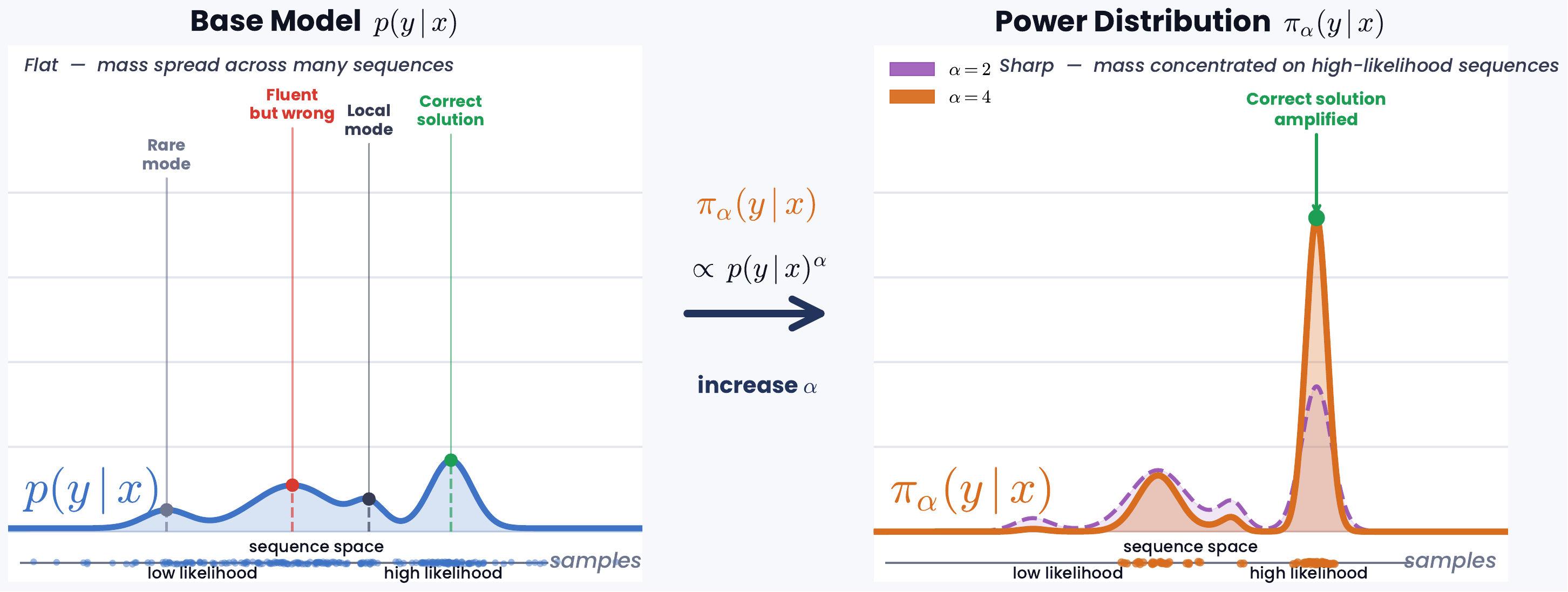}
\vspace{-4mm}
\caption{Conceptual illustration of the power distribution.
\emph{Left:} the base model $p(y\mid x)$ spreads probability mass across many sequences, including fluent-but-wrong ones.
\emph{Right:} the power distribution $\pi_\alpha\propto p^\alpha$ ($\alpha{=}2,4$) concentrates mass on high-likelihood sequences.
Correct solutions are amplified while lower-likelihood sequences are suppressed.
Power-SMC samples sequences from $\pi_\alpha$ without modifying model parameters.}
\vspace{-5mm}
\label{fig:concept}
\end{figure}
\citet{karan2026reasoning} propose using Metropolis--Hastings (MH) sampling to draw samples from~\eqref{eq:power}.
They show that this training-free strategy can match RL-based post-training on reasoning benchmarks.
However, MH is expensive for LLM inference.
Each MH step picks a position in the sequence and regenerates all tokens after that position.
These steps are inherently sequential: each accept/reject decision must finish before the next begins.
This serial structure dominates wall-clock time, even with standard Transformer caching.
The resulting high latency makes MH-based power sampling impractical to deploy.

To address these limitations, we introduce \textbf{Power-SMC}, a sampling method that targets the same distribution~\eqref{eq:power} while keeping inference \emph{batch-parallel}.

\begin{wrapfigure}{r}{0.5\linewidth}
\centering
\vspace{-1mm}
\includegraphics[width=\linewidth]{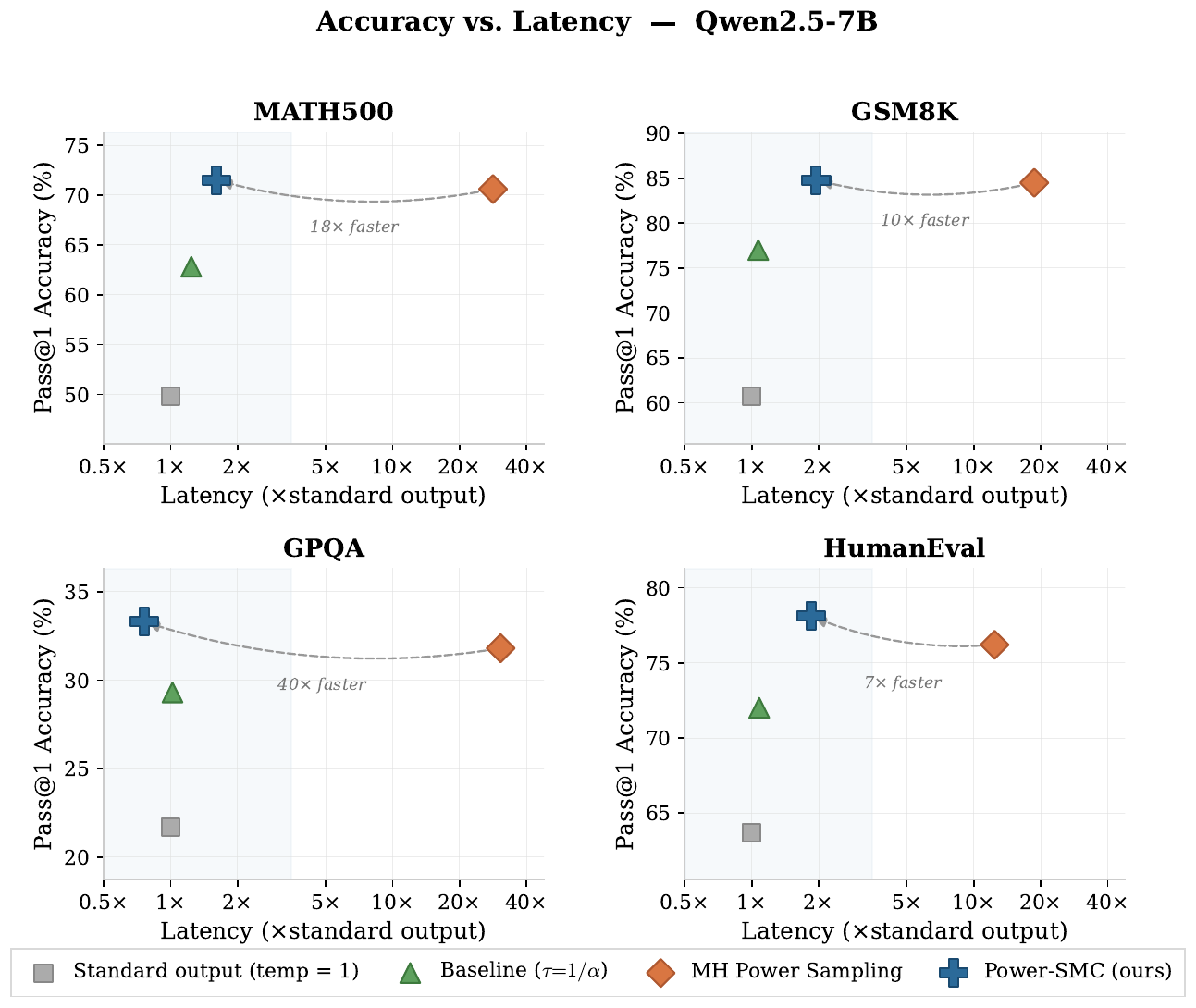}
\vspace{-4mm}
\caption{\textbf{Accuracy vs.\ latency across four benchmarks} (\texttt{Qwen2.5-7B}, $\alpha{=}4$).
Each panel plots pass@1 accuracy against wall-clock latency relative to standard output decoding; dashed arrows indicate the latency reduction from MH to Power-SMC.
Power-SMC achieves accuracy comparable to MH power sampling while being 7--40$\times$ faster.}
\label{fig:pareto}
\vspace{-4mm}
\end{wrapfigure}

The key idea is to generate $N$ candidate sequences simultaneously, each with an associated importance weight that measures how well it matches the power distribution.
We use these importance weights to decide which candidates to keep and which to discard.
More precisely, we use the Sequential Monte Carlo (SMC) framework~\citep{delmoral2004,delmoral2006smc} to approximate $\pi_\alpha$.
SMC is a well-studied family of algorithms for approximating distributions that evolve over time.
In our case, the distribution evolves as tokens are decoded.
Each candidate has a \emph{prefix}, the tokens it has generated so far.
At each decoding step, we approximate $\pi_\alpha$ restricted to the prefixes generated up to that step.
SMC maintains $N$ candidate sequences in parallel, updating each candidate's importance weight as tokens are decoded.
When the importance weights become uneven and a few candidates dominate, SMC performs a \emph{resampling} step on the candidates.
It duplicates high-scoring candidates and discards low-scoring ones, then continues decoding.

\textbf{Our contributions.} 
We present \textbf{Power-SMC} (Section~\ref{sec:method}), a training-free sampling method that elicits inference-time reasoning capabilities from LLMs.
Power-SMC produces samples from the power distribution one token at a time, using a sequence of intermediate targets, one for each prefix length.
At each decoding step, it samples the next token from a \emph{proposal distribution}: any distribution that depends only on the prompt and the tokens generated so far.
It then applies an \emph{incremental importance weight} that corrects for the mismatch between the proposal and the power distribution.
When a few candidate sequences dominate the importance weights, Power-SMC resamples them.
Because all candidate sequences decode in a single batch, the method is GPU-friendly.
We also introduce a gradual sharpening schedule called $\alpha$-ramping.
This schedule slowly increases the exponent from 1 to its target value $\alpha$, improving weight stability while still targeting the same power distribution.

We further prove that, among all such proposal distributions, sampling temperature $\tau=1/\alpha$ uniquely eliminates per-step weight variance (Section~\ref{sec:theory}).
We characterize the remaining source of weight instability.
Differences in next-token uncertainty across candidate sequences cause their importance weights to diverge over time, even under this optimal temperature.

Finally, we analyze the computational cost of Power-SMC versus MH (Section~\ref{sec:experiments}, Appendix~\ref{app:cost}).
We prove that MH has an inherent latency floor. We evaluate
on MATH500, GSM8K, GPQA, and HumanEval, to find Power-SMC matching or exceeding MH power sampling on 13 out of 16 model--benchmark configurations.
It reduces latency from 10--43$\times$ to 0.8--3.3$\times$ relative to standard output decoding (Figure~\ref{fig:pareto}).
Notably, its pass@$k$ curves match those of MH power sampling across all values of $k$, while RL fine-tuned models stop benefiting from additional samples.
%
% ============================================================
\section{Related Work}
\label{sec:related}
% ============================================================

\paragraph{Power sampling and MCMC for LLMs.}
\citet{karan2026reasoning} introduce MH-based sampling from the sequence-level power distribution~\eqref{eq:power} and show strong empirical gains on reasoning tasks.
MH targets the correct distribution under standard conditions, but the loop is inherently sequential: each step picks a position in the sequence and regenerates all tokens after that position.

\paragraph{Token-level approximations to power sampling.}
A recent direction avoids iterative MCMC by approximating the next-token distribution implied by the power distribution.
\citet{ji2026scalable} show that this next-token distribution can be written as a scaled low-temperature distribution, where the scaling factor depends on the likelihood of \emph{future} continuations.
Their algorithm estimates this factor by sampling possible continuations (rollouts) and applying a jackknife correction.
Power-SMC takes a different approach: our proposal depends only on the tokens generated so far, with no future tokens simulated.
We prove that temperature $\tau=1/\alpha$ is the unique choice within this class of proposals that eliminates per-step weight variance.
Any remaining gap to the power distribution is corrected through importance weighting and resampling across candidates.

\paragraph{Sequential Monte Carlo.}
SMC methods approximate evolving distributions by maintaining a population of weighted samples and periodically resampling to prevent weight collapse~\citep{doucet2001smc,delmoral2004,delmoral2006smc,doucet2011tutorial}.
Power-SMC adapts this framework to autoregressive decoding with Transformer KV caches.
This adaptation introduces a practical requirement: when a candidate is duplicated during resampling, its cached model state must be correctly copied (Appendix~\ref{sec:systems}).

\paragraph{Decoding heuristics.}
Common decoding strategies include temperature scaling, top-$k$, nucleus sampling~\citep{holtzman2020curious}, and top-H sampling~\citep{potraghloo2025top}.
These are local token-level rules and do not, in general, draw samples from global sequence-level targets like~\eqref{eq:power}.
Our results clarify the role of temperature within an algorithm that \emph{does} target the sequence-level power distribution.

% ============================================================
\section{Background}
\label{sec:background}
% ============================================================

\subsection{Autoregressive models and the power distribution}
Given a prompt $x$, a pretrained autoregressive model defines
\begin{equation}
p_\theta(y\mid x)=\prod_{t=1}^{T(y)} p_\theta(y_t\mid x,y_{<t}),
\label{eq:base}
\end{equation}
over sequences $y=(y_1,\dots,y_{T(y)})$ that end with a special end-of-sequence token ($\EOS$).
Each token $y_t$ is drawn from a vocabulary~$\vocab$.
For $\alpha\ge 1$, the power distribution~\eqref{eq:power} sharpens the base model's output distribution.
The sharpening is achieved by exponentiating the \emph{sequence-level} probability, not the individual token probabilities.

\subsection{Why token-level temperature does not match power sampling}
\label{sec:bg-temp}
A simple approach to sharpening is to lower the sampling temperature at each token:
\begin{equation}
q_t(\cdot\mid x,y_{<t})\propto p_\theta(\cdot\mid x,y_{<t})^{1/\tau}.
\end{equation}
Even when $\tau=1/\alpha$, the resulting joint distribution $q(y\mid x)=\prod_t q_t(y_t\mid x,y_{<t})$ typically differs from $\pi_\alpha(y\mid x)$.
The reason is normalization.
Temperature sampling normalizes the distribution at each token, so each $q_t$ sums to one.
The power distribution normalizes only once, over complete sequences~\citep{karan2026reasoning}.
Power-SMC addresses this gap.
It samples tokens using a proposal distribution, then applies an exact incremental importance weight at each step to target the power distribution $\pi_\alpha$.

\subsection{Importance sampling and Sequential Monte Carlo}
\label{sec:bg-is}

We review standard importance sampling and SMC in Appendix~\ref{app:bg-is-smc}.
Here we state the key diagnostic used by Power-SMC.
SMC maintains $N$ candidate sequences with associated importance weights (the candidates are referred to as \emph{particles} in the SMC literature).
At each decoding step, each candidate's importance weight is updated to measure how well it matches the power distribution.
When the importance weights become too uneven, meaning a few candidates dominate, SMC performs a \emph{resampling} step on the candidates.
It duplicates candidates with high importance weight and discards those with low importance weight, then resets all importance weights.
The standard diagnostic for detecting imbalance in the importance weights is the \emph{effective sample size} (ESS):
\begin{equation}
\ESS_t \;=\; \left(\sum_{i=1}^N (W_t^{(i)})^2\right)^{-1},
\qquad
W_t^{(i)}=\frac{\tilde W_t^{(i)}}{\sum_j \tilde W_t^{(j)}},
\label{eq:ess}
\end{equation}
where $\tilde W_t^{(i)}$ denotes the unnormalized importance weight of the $i$-th candidate at step~$t$.

\section{Power-SMC: A Batch-Parallel Sampling Algorithm}
\label{sec:method}
% ============================================================

\subsection{Targeting the power distribution one token at a time}
Power-SMC targets the power distribution $\pi_\alpha$ indirectly.
Instead of trying to draw a sample from $\pi_\alpha$ in one step, it defines a sequence of intermediate targets, one for each prefix length.
Power-SMC maintains $N$ candidate sequences in parallel.
At each decoding step, every candidate sequence is extended by one token, approximating the intermediate target for that prefix length.
This type of decomposition is standard in sequential inference and is known as the Feynman--Kac framework~\citep{delmoral2004}.
Let $y_{1:t}$ denote a prefix of length~$t$.
We define the unnormalized prefix target as
\begin{equation}
\gamma_t(y_{1:t}\mid x)\;:=\;p_\theta(y_{1:t}\mid x)^\alpha,
\qquad
p_\theta(y_{1:t}\mid x)=\prod_{s=1}^t p_\theta(y_s\mid x,y_{<s}),
\label{eq:gamma}
\end{equation}
and let $\pi_t=\gamma_t/Z_t$ be the corresponding normalized distribution.
A candidate stops generating once it emits the end-of-sequence token, and its importance weight is frozen.
The distribution over completed sequences then matches the power distribution $\pi_\alpha$ (Appendix~\ref{app:eos}).

\paragraph{Sampling tokens and updating importance weights.}
At each step, Power-SMC samples the next token for each candidate from a \emph{proposal distribution} $q_t(\cdot\mid x,y_{<t})$.
The proposal can be any distribution over the next token that depends only on the prompt and the tokens generated so far.
It does not need to match the next-token distribution implied by $\pi_\alpha$; it only needs to assign nonzero probability to every token the model might generate.

To correct for this mismatch, each candidate's importance weight is multiplied at each step by a correction factor called the \emph{incremental importance weight}:
\begin{equation}
\omega_t(y_{1:t})
\;=\;
\frac{\gamma_t(y_{1:t}\mid x)}{\gamma_{t-1}(y_{1:t-1}\mid x)\,q_t(y_t\mid x,y_{<t})}
\;=\;
\boxed{
\frac{p_\theta(y_t\mid x,y_{<t})^\alpha}{q_t(y_t\mid x,y_{<t})}.
}
\label{eq:inc}
\end{equation}
A candidate's importance weight $\tilde W_t^{(i)}$ is therefore the product of all incremental importance weights $\omega_1^{(i)} \cdots \omega_t^{(i)}$ accumulated by that candidate sequence.
Intuitively, when the proposal samples a token that has high probability under the power distribution, the incremental factor is large and the candidate's importance weight grows.
When it samples an unlikely token, the importance weight shrinks.
Together, the candidate sequences and their importance weights form an approximation of the power distribution.

\subsection{The full Power-SMC procedure}
Power-SMC decodes $N$ candidate sequences in parallel.
After each token, it multiplies each candidate's importance weight by its incremental importance weight from Eq.~\eqref{eq:inc}.
If the effective sample size (ESS) drops below a threshold, the candidates are resampled.
Resampling duplicates candidates with high importance weight, discards those with low importance weight, and resets all importance weights to be equal.
This concentrates future computation on the candidates that dominate the importance weights.
The final output is a single sequence, drawn at random from the candidates with probability proportional to their final importance weights.
The complete pseudocode is given in Algorithm~\ref{alg:smc} (Appendix~\ref{app:algorithm}).

% ============================================================
\section{Optimal Temperature and Importance Weight Stability}
\label{sec:theory}
% ============================================================

This section addresses two questions.
First, what is the best proposal distribution to use inside Power-SMC, and what guarantees does it provide?
Second, what source of weight instability remains under this best choice, and how can it be mitigated?

\subsection{Best proposal distribution}

Fix a decoding step $t$ and suppose we have already generated a prefix $y_{<t}$.
Let $p_t(v):=p_\theta(v\mid x,y_{<t})$ denote the model's next-token probabilities.
For any choice of proposal distribution $q_t(v)$ (as defined in Section~\ref{sec:method}), the incremental importance weight for sampling token $v$ is
\begin{equation}
\omega_t(v;\, y_{<t})=\frac{p_t(v)^\alpha}{q_t(v)}.
\label{eq:omega}
\end{equation}

\begin{theorem}[Optimal proposal for Power-SMC]
\label{thm:localopt}
Fix $t$ and a prefix $y_{<t}$.
Among all proposal distributions $q_t$ that assign positive probability wherever the model does ($q_t(v)>0$ whenever $p_t(v)>0$), the unique minimizer of the conditional second moment
$\E_{v\sim q_t}\!\bigl[\omega_t(v;\, y_{<t})^2\bigr]$
(and hence of the conditional variance) is
\begin{equation}
\boxed{
q_t^\star(v\mid x,y_{<t})=\frac{p_t(v)^\alpha}{\sum_u p_t(u)^\alpha}.
}
\label{eq:qstar}
\end{equation}
Under $q_t^\star$, the incremental weight is deterministic given the prefix:
$\omega_t(v;\, y_{<t})\equiv \sum_u p_t(u)^\alpha$ for all~$v$, so
$\Var_{q_t^\star}\!\bigl(\omega_t(\cdot;\, y_{<t})\bigr)=0$.
\end{theorem}

\begin{corollary}[Temperature form]
\label{cor:temp}
If $p_t=\softmax(\ell_t)$ for logits $\ell_t$, then $q_t^\star(v)\propto \exp(\alpha\,\ell_t(v))=\softmax(\ell_t/\tau)$ with $\tau=1/\alpha$.
\end{corollary}

\paragraph{Interpretation.}
Corollary~\ref{cor:temp} does \emph{not} claim that temperature sampling at $\tau=1/\alpha$ alone produces exact samples from the sequence-level power distribution~\eqref{eq:power}.
Rather, the claim is the following.
If we restrict ourselves to proposal distributions (those that depend only on the tokens generated so far, with no reliance on future tokens), then $\tau=1/\alpha$ is the unique choice that makes the incremental importance weight as stable as possible.
It eliminates variance due to which token was sampled.
The remaining variance comes entirely from which sequence a candidate is generating.

\subsection{Characterizing the remaining importance weight instability}
\label{sec:renyi}

Theorem~\ref{thm:localopt} shows that the optimal proposal $q_t^\star$ (temperature $\tau=1/\alpha$) makes the incremental importance weight independent of which token was sampled.
However, the incremental importance weight still depends on the prefix.
Specifically, at step $t$ it equals $Z_t(\alpha;\, y_{<t}) := \sum_v p_t(v)^\alpha$.
This quantity measures how concentrated the model's next-token predictions are at that prefix.
When the model is very confident about the next token (one token has most of the probability), $Z_t$ is large.
When the model is uncertain (probability spread across many tokens), $Z_t$ is small.

This quantity is directly related to the \emph{R\'enyi entropy} $H_\alpha$, a generalization of Shannon entropy that quantifies how spread out a probability distribution is.
Lower entropy means more concentrated.
The connection is $\log Z_t = (1{-}\alpha)\,H_\alpha(p_\theta(\cdot\mid x,y_{<t}))$.
Before any resampling, the total accumulated log-importance-weight for the $i$-th candidate is therefore
\begin{equation}
\log \tilde W_{T_{\max}}^{(i)}
\;=\;
(1-\alpha)\sum_{t=1}^{T_{\max}} H_\alpha\!\bigl(p_\theta(\cdot\mid x,y_{<t}^{(i)})\bigr),
\label{eq:pathweight}
\end{equation}
which is non-positive for $\alpha \ge 1$.
The interpretation is intuitive.
Candidates whose generated prefixes lead through highly uncertain next-token predictions (high R\'enyi entropy) accumulate lower importance weights.
Candidates whose generated prefixes are more confidently predicted by the model accumulate higher importance weights.

This is the remaining source of importance weight collapse.
Even with the best per-step choice of temperature, importance weights diverge when different candidates encounter prefixes with substantially different levels of next-token uncertainty.
See Appendix~\ref{app:renyi} for the full derivation.

\subsection{Gradual sharpening ($\alpha$-ramping)}
\label{sec:ramping}

Equation~\eqref{eq:pathweight} shows that importance weight collapse is driven by differences in next-token uncertainty across candidates.
Applying the full exponent $\alpha$ from the first token amplifies these differences immediately, before candidates have explored diverse reasoning strategies.

To address this, we increase $\alpha$ gradually.
We define a schedule of intermediate values $1{=}\alpha^{(0)}{<}\alpha^{(1)}{<}\cdots{<}\alpha^{(L)}{=}\alpha$, starting from no sharpening and ending at the target exponent.
At each stage boundary, the importance weights are updated to transition from one intermediate target to the next:
$\log \tilde W \leftarrow \log \tilde W + (\alpha^{(\ell)}{-}\alpha^{(\ell-1)})\cdot \log p_\theta(y_{1:t}\mid x)$.
These updates telescope, so the final target is exactly $\pi_\alpha$, the same distribution as without ramping (Appendix~\ref{app:renyi}).
Within each stage, the best temperature adapts to $\tau_\ell = 1/\alpha^{(\ell)}$.
In our experiments, we use a simple linear schedule over the first $T_{\mathrm{ramp}}$ tokens.

% ============================================================
\section{Experiments}
\label{sec:experiments}
% ============================================================

\paragraph{Cost framing.}
Power-SMC performs $N{\cdot}T$ forward passes across $T$ decode steps, all in parallel as a single batch, so modern GPU hardware can absorb most of the additional compute.
MH power sampling is inherently sequential: each accept/reject move regenerates all tokens after the chosen edit position.
We prove a lower bound on MH overhead.
Even under the most favorable editing strategy, where each move modifies only the most recent segment of $B$ tokens, MH incurs at least $(1{+}M/2)\times$ the cost of standard decoding.
Here $M$ is the number of moves per segment.
For $M{=}10$ this gives at least $6\times$ overhead, a floor that no implementation can go below.
More typical MH variants allow moves to target any prefix position.
In this case, expected cost scales as $\Theta(T^2/B)$, specifically $1+M(K{+}1)/4$ where $K{=}T/B$.
For our setting ($T_{\max}{=}2048$, $B{=}64$, $M{=}10$) this gives $\approx 83\times$ overhead, an order of magnitude above the $6\times$ lower bound.
Full derivations are in Appendix~\ref{app:cost}.

\paragraph{Setup.}
We evaluate on MATH500, GSM8K~\citep{cobbe2021gsm8k}, GPQA~\citep{rein2023gpqa}, and HumanEval~\citep{chen2021humaneval} across four open-weight models: \texttt{Qwen2.5-7B}, \texttt{Qwen2.5-Math-7B}~\citep{qwen2.5}, \texttt{Phi-3.5-mini-instruct}~\citep{abdin2024phi3}, and \texttt{Qwen3-4B}~\citep{qwen3}. We compare:
(i)~standard output decoding (temperature 1),
(ii)~baseline decoding at low temperature $\tau=1/\alpha$,
(iii)~MH power sampling~\citep{karan2026reasoning},
(iv)~Scalable Power Sampling~\citep{ji2026scalable}, and
(v)~Power-SMC.
We measure end-to-end wall-clock latency on identical hardware using Hugging Face (no specialized inference engine) and report accuracy--latency trade-offs.

\paragraph{Implementation details.}
Unless otherwise noted, we use $N=64$ candidate sequences (particles) and exponent $\alpha=4$.
We check for resampling every $B=64$ tokens.
The maximum generation length is $T_{\max}=2048$ tokens.
Resampling is triggered when the effective sample size drops below half the number of candidates: $\ESS_t < 0.5 \cdot N$.
We apply $\alpha$-ramping with a linear schedule over the first $T_{\mathrm{ramp}}=100$ tokens.
For resampling, we use \emph{systematic resampling}, a standard low-variance method that selects which candidates to duplicate using a single random offset together with the cumulative importance weights.
After resampling, candidate prefixes are copied and the Transformer's cached states (KV cache) are reordered accordingly (Appendix~\ref{sec:systems}).
All importance weights are then reset to be equal.
Two optimizations control memory.
First, the prompt is processed once at batch size~1 and the resulting cache is replicated to all $N$ candidate sequences, avoiding redundant prompt-phase compute.
Second, after resampling, only $U$ unique ancestors typically remain ($U{\approx}14$ on average for $N{=}64$).
A copy-on-write strategy retains only the $U$ physical caches, deferring expansion until candidate sequences diverge.
This reduces transient peak memory from $2N$ to $N{+}U$ cache equivalents.
The result is a peak GPU memory of 18.38\,GB on \texttt{Qwen2.5-7B}.
This is only 2.9\,GB above single-sequence decoding and well within the capacity of a 24\,GB consumer GPU (Appendix~\ref{app:memory}).

\begin{table*}[t]
\centering
\small
\setlength{\tabcolsep}{4pt}
\renewcommand{\arraystretch}{1.15}

\resizebox{\textwidth}{!}{%
\begin{tabular}{llccccc}
\toprule
& & \textbf{MATH500} & \textbf{GSM8K} & \textbf{GPQA} & \textbf{HumanEval} & \textbf{Latency}$\downarrow$ \\
\textbf{Model} & \textbf{Method} & \textbf{Acc.\,(\%)}$\uparrow$ & \textbf{Acc.\,(\%)}$\uparrow$ & \textbf{Acc.\,(\%)}$\uparrow$ & \textbf{Acc.\,(\%)}$\uparrow$ & \textbf{(MATH500)} \\
\midrule
\multirow{5}{*}{\texttt{Qwen2.5-7B}}
& Standard output & 49.8 & 60.7 & 21.7 & 64.0 & 1.00$\times$ \\
& Baseline ($\tau{=}1/\alpha$) & 62.8 & 77.0 & 29.3 & 72.0 & 1.24$\times$ \\
& MH power sampling & 70.6 & \underline{84.5} & 31.8 & \underline{76.2} & 28.30$\times$ \\
& Scalable Power Sampling$^{\dagger}$ & \underline{70.8} & 82.5 & \underline{32.8} & 75.6 & 2.5--3.5$\times$ \\
& \cellcolor[HTML]{D5E8D4}\textbf{Power-SMC (ours)}$^{\ddagger}$ & \cellcolor[HTML]{D5E8D4}\textbf{71.5} & \cellcolor[HTML]{D5E8D4}\textbf{84.8} & \cellcolor[HTML]{D5E8D4}\textbf{33.3} & \cellcolor[HTML]{D5E8D4}\textbf{78.1} & \cellcolor[HTML]{D5E8D4}\textbf{1.61$\times$} \\
\midrule
\multirow{5}{*}{\texttt{Qwen2.5-Math-7B}}
& Standard output & 49.6 & 62.2 & 23.2 & 32.9 & 1.00$\times$ \\
& Baseline ($\tau{=}1/\alpha$) & 69.0 & 77.3 & 27.8 & 52.4 & 1.00$\times$ \\
& MH power sampling & 74.8 & 81.5 & \underline{33.8} & \underline{62.2} & 18.20$\times$ \\
& Scalable Power Sampling$^{\dagger}$ & \underline{75.8} & \underline{83.2} & \textbf{37.4} & 60.4 & 2.5--3.5$\times$ \\
& \cellcolor[HTML]{D5E8D4}\textbf{Power-SMC (ours)}$^{\ddagger}$ & \cellcolor[HTML]{D5E8D4}\textbf{76.1} & \cellcolor[HTML]{D5E8D4}\textbf{87.6} & \cellcolor[HTML]{D5E8D4}33.3 & \cellcolor[HTML]{D5E8D4}\textbf{64.0} & \cellcolor[HTML]{D5E8D4}\textbf{1.46$\times$} \\
\midrule
\multirow{4}{*}{\texttt{Phi-3.5-mini-instruct}}
& Standard output & 40.0 & 85.5 & 21.2 & 51.8 & 1.00$\times$ \\
& Baseline ($\tau{=}1/\alpha$) & 47.8 & 88.4 & 23.2 & \underline{65.9} & 0.91$\times$ \\
& MH power sampling & \underline{50.8} & \underline{89.3} & \textbf{32.3} & 62.8 & 16.12$\times$ \\
& \cellcolor[HTML]{D5E8D4}\textbf{Power-SMC (ours)}$^{\ddagger}$ & \cellcolor[HTML]{D5E8D4}\textbf{51.4} & \cellcolor[HTML]{D5E8D4}\textbf{89.9} & \cellcolor[HTML]{D5E8D4}\underline{31.2} & \cellcolor[HTML]{D5E8D4}\textbf{66.5} & \cellcolor[HTML]{D5E8D4}\textbf{3.14$\times$} \\
\midrule
\multirow{4}{*}{\texttt{Qwen3-4B}}
& Standard output & 73.6 & 88.1 & 19.2 & \underline{46.8} & 1.00$\times$ \\
& Baseline ($\tau{=}1/\alpha$) & 74.0 & 86.5 & 25.3 & 40.8 & 0.97$\times$ \\
& MH power sampling & \underline{76.2} & \underline{90.6} & \textbf{36.4} & 45.7 & 19.34$\times$ \\
& \cellcolor[HTML]{D5E8D4}\textbf{Power-SMC (ours)}$^{\ddagger}$ & \cellcolor[HTML]{D5E8D4}\textbf{77.8} & \cellcolor[HTML]{D5E8D4}\textbf{91.9} & \cellcolor[HTML]{D5E8D4}\underline{35.4} & \cellcolor[HTML]{D5E8D4}\textbf{50.6} & \cellcolor[HTML]{D5E8D4}\textbf{1.55$\times$} \\
\bottomrule
\end{tabular}%
}

\caption{\textbf{Pass@1 accuracy (\%) and wall-clock latency across four benchmarks and four models.}
Latency is normalized to standard output for each model ($1.00\times$).
The values shown are for MATH500; per-benchmark variation is small (Figure~\ref{fig:pareto}).
\textbf{Bold} and \underline{underline} indicate the highest and second-highest accuracy among all methods for each model--benchmark pair.
$^{\dagger}$\citet{ji2026scalable} report 2.5--3.5$\times$ overhead for their rollout-based configuration; we list their reported range since they do not provide model-specific latencies in our evaluation stack.
$^{\ddagger}$Results are the mean over 3 independent runs; standard deviation $\le$0.35\% for accuracy and $\le$0.11$\times$ for latency.
GPQA accuracies are computed under a strict multiple-choice parser applied uniformly to all methods (see Appendix~\ref{app:grpo_bon}).
}
\label{tab:results}
\end{table*}

\paragraph{Results.}
Table~\ref{tab:results} reports pass@1 accuracy and latency across all four benchmarks and four models.
Power-SMC achieves the highest accuracy among all training-free methods on 13 out of 16 model--benchmark configurations.
The three exceptions are all on GPQA.
On \texttt{Phi-3.5} and \texttt{Qwen3-4B}, MH leads Power-SMC by at most 1.1 points.
On \texttt{Qwen2.5-Math-7B}, Scalable Power Sampling leads Power-SMC by 4.1 points.
The gains are most pronounced on GSM8K, where Power-SMC outperforms MH by up to 6.1 points (\texttt{Qwen2.5-Math-7B}: 87.6\% vs.\ 81.5\%).
Baseline low-temperature decoding ($\tau=1/\alpha$) captures a substantial portion of the accuracy gain at near-zero overhead, but consistently falls short of sequence-level methods.

On latency, Power-SMC operates at 0.8--3.3$\times$ standard output across all benchmarks, while MH incurs 10--43$\times$ overhead.
Per-benchmark latencies are reported in Table~\ref{tab:latency_full}.
\citet{ji2026scalable} report 2.5--3.5$\times$ overhead for Scalable Power Sampling, placing it in a different latency regime from Power-SMC.
Table~\ref{tab:grpo_bon} additionally compares Power-SMC against GRPO~\citep{shao2024deepseekmath} and Best-of-$N$ sampling on \texttt{Qwen2.5-7B} and \texttt{Qwen2.5-Math-7B}.
Our goal is not to outperform GRPO, an RL-based post-training method that modifies model parameters, but to show that a training-free method can come close.
Power-SMC outperforms Best-of-$N$ on all six configurations and closes most of the gap to GRPO on MATH500 and GPQA (within 3.0 points).
On HumanEval, where output diversity matters for finding correct programs, Power-SMC \emph{exceeds} GRPO by 10.3--22.0 points without modifying model parameters.

\begin{table}[t]
\centering
\small
\setlength{\tabcolsep}{6pt}
\renewcommand{\arraystretch}{1.15}
\begin{tabular}{llccc}
\toprule
\textbf{Model} & \textbf{Method} & \textbf{MATH500} & \textbf{HumanEval} & \textbf{GPQA} \\
\midrule
\multirow{3}{*}{\texttt{Qwen2.5-7B}}
& GRPO~\citep{shao2024deepseekmath} & \textbf{74.0} & 56.1 & \textbf{35.4} \\
& Best-of-$N$ ($N{=}64$) & 67.0 & 65.9 & 30.3 \\
& \cellcolor[HTML]{D5E8D4}\textbf{Power-SMC (ours)} & \cellcolor[HTML]{D5E8D4}71.5 & \cellcolor[HTML]{D5E8D4}\textbf{78.1} & \cellcolor[HTML]{D5E8D4}33.3 \\
\midrule
\multirow{3}{*}{\texttt{Qwen2.5-Math-7B}}
& GRPO~\citep{shao2024deepseekmath} & \textbf{78.0} & 53.7 & \textbf{36.3} \\
& Best-of-$N$ ($N{=}64$) & 71.4 & 59.7 & 32.3 \\
& \cellcolor[HTML]{D5E8D4}\textbf{Power-SMC (ours)} & \cellcolor[HTML]{D5E8D4}76.1 & \cellcolor[HTML]{D5E8D4}\textbf{64.0} & \cellcolor[HTML]{D5E8D4}33.3 \\
\bottomrule
\end{tabular}
\vspace{8pt}
\caption{\textbf{Comparison with GRPO and Best-of-$N$ (pass@1 accuracy, \%).}
GRPO uses the post-trained models from~\citet{karan2026reasoning}.
Best-of-$N$ draws $N{=}64$ independent completions from the base model and selects the one with the highest base-model log-likelihood, matching the candidate count used by Power-SMC.
\textbf{Bold} indicates the highest accuracy in each column for each model.
Setup details and the GPQA parser note are in Appendix~\ref{app:grpo_bon}.}
\label{tab:grpo_bon}
\end{table}

\paragraph{Diversity: pass@$k$ evaluation.}
\begin{wrapfigure}{r}{0.5\linewidth}
\centering
\vspace{-4mm}
\includegraphics[width=\linewidth]{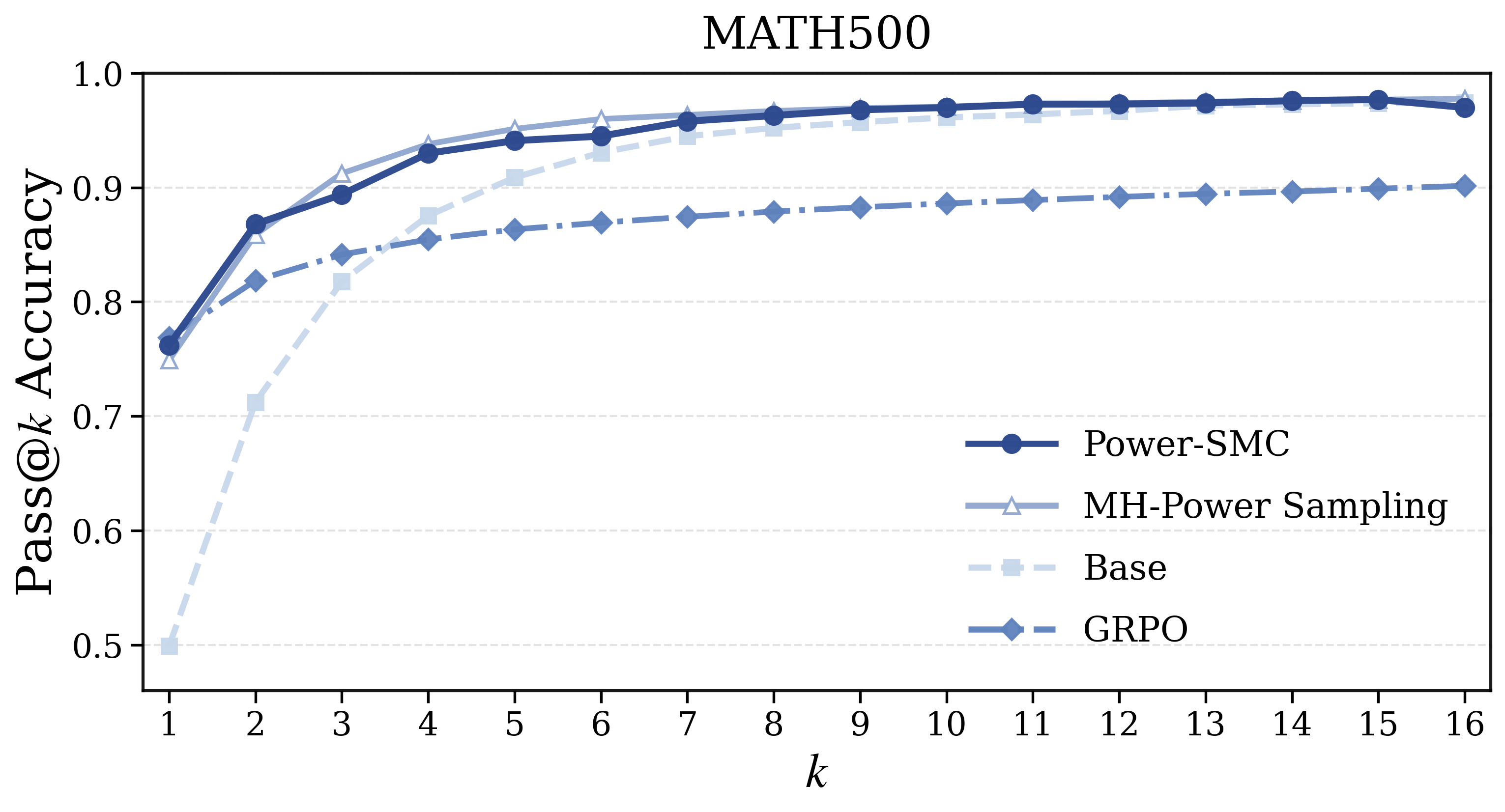}
\vspace{-6mm}
\caption{\textbf{Pass@$k$ accuracy on MATH500} (\texttt{Qwen2.5-Math-7B}, $\alpha{=}4$).
Power-SMC and MH power sampling track closely, confirming equivalent diversity under the same target distribution.
GRPO saturates early, illustrating the diversity loss caused by permanently modifying model parameters.}
\label{fig:passk}
\vspace{-2mm}
\end{wrapfigure}
A key concern with distribution sharpening is diversity collapse.
If the sampler concentrates on a single mode, drawing multiple samples produces nearly the same output every time.
To assess this, we evaluate pass@$k$ accuracy on MATH500 using \texttt{Qwen2.5-Math-7B} for $k\in\{1,2,\dots,16\}$.
For each method, we draw $k$ independent samples per problem.
Figure~\ref{fig:passk} reports results for Power-SMC, MH power sampling~\citep{karan2026reasoning}, GRPO~\citep{shao2024deepseekmath}, and standard output decoding.

Power-SMC and MH power sampling produce nearly identical pass@$k$ curves across the entire range.
This confirms that both methods target the same sequence-level power distribution and produce comparably diverse solutions.
Both start well above standard output at $k{=}1$ and continue climbing through $k{=}16$, reaching approximately $98\%$.
Standard output starts lower but rises steeply, consistent with its broader but less concentrated distribution.

GRPO tells a different story.
Its pass@$k$ curve flattens around $k{=}7$--$8$ and plateaus near $90\%$, substantially below the power-sampling methods.
This is consistent with the diversity collapse noted by \citet{karan2026reasoning}.
RL fine-tuning sharpens the distribution by \emph{modifying parameters}, which permanently limits the range of outputs the model can produce.
Power-SMC achieves the same sharpening at inference time without permanently narrowing the model's range of outputs.

\texttt{Qwen3-4B} is a reasoning-tuned model that benefits from extended chain-of-thought generation.
For this model, aggressive sharpening can be counterproductive.
Applying high $\alpha$ from the first token tends to wrap up reasoning quickly, short-circuiting the chain of thought that a reasoning-tuned model is designed to produce.
Baseline low-temperature decoding underperforms standard output on GSM8K (86.5\% vs.\ 88.1\%) and HumanEval (40.8\% vs.\ 46.8\%).
MH power sampling similarly underperforms standard output on HumanEval (45.7\% vs.\ 46.8\%).
Power-SMC remains effective on this model, improving over standard output on all benchmarks.
The gradual $\alpha$-ramping schedule preserves the base distribution for the first $T_{\mathrm{ramp}}$ tokens.
This lets the model produce its trained reasoning style before sharpening concentrates probability mass on the most promising continuations.

\begin{figure}[t]
\centering
\includegraphics[width=0.85\linewidth]{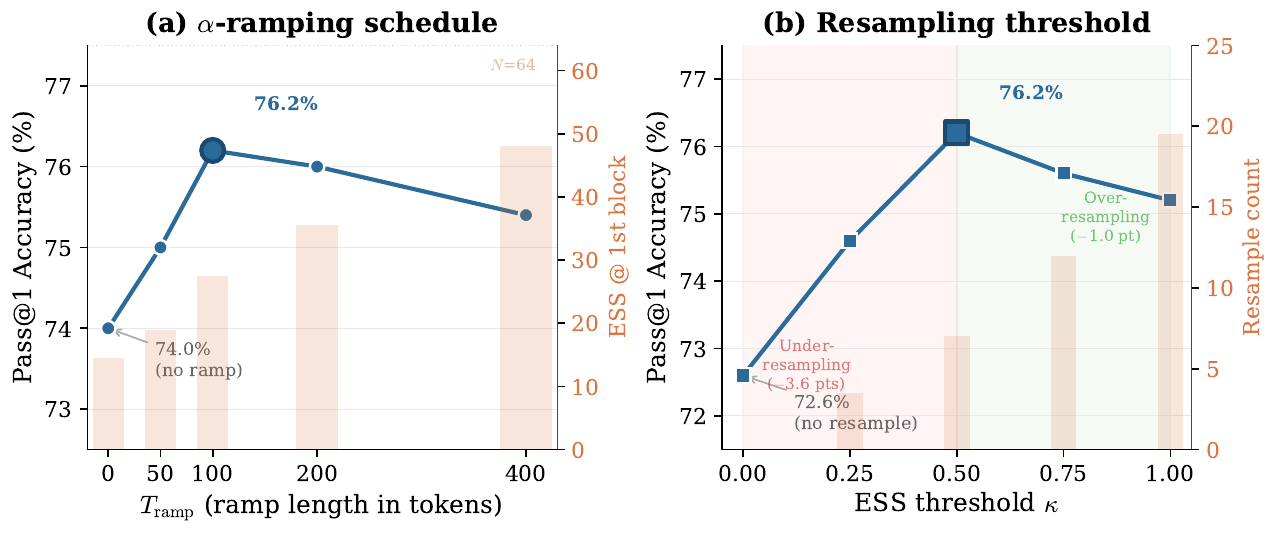}
\vspace{-2mm}
\caption{\textbf{Ablation studies} on MATH500 (\texttt{Qwen2.5-Math-7B}, $\alpha{=}4$, $N{=}64$).
\emph{(a)}~Effect of $\alpha$-ramp length: accuracy peaks at $T_{\mathrm{ramp}}{=}100$; bars show ESS at the first resampling check point (right axis).
\emph{(b)}~Effect of ESS threshold $\kappa$: disabling resampling ($\kappa{=}0$) causes a 3.6-point drop; bars show total resample count (right axis). 
Under-resampling is substantially more damaging than over-resampling.}
\vspace{-2mm}
\label{fig:ablations}
\end{figure}

\paragraph{Ablation studies.}
We ablate the two main hyperparameters on MATH500 with \texttt{Qwen2.5-Math-7B}.
The first is the $\alpha$-ramp length $T_{\mathrm{ramp}}$.
The second is the ESS resampling threshold $\kappa$.
Full tables are in Appendix~\ref{app:ablations}.
Figure~\ref{fig:ablations} summarizes the results.

For $\alpha$-ramping (panel~a), accuracy peaks at $T_{\mathrm{ramp}}{=}100$ (76.2\%).
Without ramping ($T_{\mathrm{ramp}}{=}0$), accuracy drops to 74.0\%, a 2.2-point decrease.
This confirms that gradually increasing the exponent mitigates the early weight collapse predicted by the R\'enyi entropy analysis in Section~\ref{sec:renyi}.
The secondary axis shows that ESS at the first resampling check (after 64 tokens) grows steadily with ramp length, reflecting reduced early weight divergence.
However, too-slow ramping ($T_{\mathrm{ramp}}{=}400$) delays the onset of sharpening and slightly hurts accuracy.

For the ESS threshold (panel~b), disabling resampling entirely ($\kappa{=}0$) reduces accuracy to 72.6\%.
This is a 3.6-point gap compared to $\kappa{=}0.5$, confirming that resampling is essential for redirecting computation toward the most promising candidates.
Without resampling, the path-level weight divergence in Eq.~\eqref{eq:pathweight} accumulates unchecked.
A few candidate sequences absorb the vast majority of the mass, and the final draw effectively reduces to picking a single sequence.
The degradation is notably asymmetric.
Under-resampling ($-$3.6 points) is far more damaging than over-resampling ($\kappa{=}1.0$: $-$1.0 point).
This suggests that practitioners should lean toward more frequent resampling.

% ============================================================
\section{Conclusion}
\vspace{-2mm}
\label{sec:conclusion}
% ============================================================

We introduced Power-SMC, a low-latency sampling method for the sequence-level power distribution $\pi_\alpha(y\mid x)\propto p_\theta(y\mid x)^\alpha$ that decodes all candidates in parallel and avoids the serial accept/reject structure of MH.
We proved that, among all proposal distributions that do not rely on future tokens, sampling temperature $\tau=1/\alpha$ uniquely eliminates per-step weight variance, characterized the remaining source of weight instability through R\'enyi entropy, and introduced a gradual sharpening schedule to mitigate it.
On the practical side, we showed how to correctly duplicate cached model states during resampling in Transformer inference and proved a lower bound on MH latency overhead.
Empirically, Power-SMC matches or exceeds MH power sampling across four benchmarks and four model families while reducing inference latency from MH's 10--43$\times$ overhead to 0.8--3.3$\times$ relative to standard output decoding, and unlike RL-based sharpening, it preserves the diversity of the base model's outputs.

\newpage
\bibliography{neurips_2026}
\bibliographystyle{neurips_2026}

%%%%%%%%%%%%%%%%%%%%%%%%%%%%%%%%%%%%%%%%%%%%%%%%%%%%%%%%%%%%
\newpage
\appendix

\section{Proof of Theorem~\ref{thm:localopt} and Corollary~\ref{cor:temp}}
\label{app:proofs}

\begin{proof}[Proof of Theorem~\ref{thm:localopt}]
Fix a prefix $y_{<t}$ and abbreviate $p(v):=p_t(v)$ and $q(v):=q_t(v)$ over $v\in \vocab\cup\{\EOS\}$.
The incremental weight for sampling $v\sim q$ is $\omega(v)=p(v)^\alpha/q(v)$.

\emph{Step 1: the conditional mean is invariant to $q$.}
\[
\E_{v\sim q}[\omega(v)]
=\sum_v q(v)\frac{p(v)^\alpha}{q(v)}
=\sum_v p(v)^\alpha
=:Z(\alpha;\,y_{<t}).
\]

\emph{Step 2: minimize the second moment.}
\[
\E_{v\sim q}[\omega(v)^2]
=\sum_v q(v)\left(\frac{p(v)^\alpha}{q(v)}\right)^{\!2}
=\sum_v \frac{p(v)^{2\alpha}}{q(v)}.
\]
We minimize $\sum_v p(v)^{2\alpha}/q(v)$ subject to $\sum_v q(v)=1$ and $q(v)>0$ whenever $p(v)>0$.
The Lagrangian is
$
\mathcal{L}(q,\lambda)=\sum_v {p(v)^{2\alpha}}/{q(v)}+\lambda\bigl(\sum_v q(v)-1\bigr).
$
Stationarity for each $v$ yields
\[
\frac{\partial \mathcal{L}}{\partial q(v)}=-\frac{p(v)^{2\alpha}}{q(v)^2}+\lambda=0
\quad\Rightarrow\quad
q(v)=\frac{p(v)^\alpha}{\sqrt{\lambda}}.
\]
Normalization $\sum_v q(v)=1$ gives $\sqrt{\lambda}=\sum_v p(v)^\alpha=Z(\alpha;\,y_{<t})$, hence
$q^\star(v)=p(v)^\alpha\big/Z(\alpha;\,y_{<t})$, which is unique.

\emph{Step 3: verify zero conditional variance.}
Substituting into $\omega(v)=p(v)^\alpha/q^\star(v)$ yields $\omega(v)\equiv Z(\alpha;\,y_{<t})$ for all~$v$, so the conditional variance is~$0$.
\end{proof}

\begin{proof}[Proof of Corollary~\ref{cor:temp}]
If $p_t=\softmax(\ell_t)$, then $p_t(v)^\alpha \propto \exp(\alpha\, \ell_t(v))$.
Thus $q_t^\star(v)\propto p_t(v)^\alpha$ equals $\softmax(\alpha\,\ell_t)$, which is temperature sampling with $\tau=1/\alpha$ (i.e., logits divided by~$\tau$).
\end{proof}

\section{Power-SMC Pseudocode}
\label{app:algorithm}

Algorithm~\ref{alg:smc} gives the complete Power-SMC procedure described in Section~\ref{sec:method}.
The main loop runs over decoding steps; within each step, all $N$ candidate sequences advance in parallel.

\begin{algorithm}[h]
\caption{Power-SMC for $\pi_\alpha(y\mid x)\propto p_\theta(y\mid x)^\alpha$}
\label{alg:smc}
\begin{algorithmic}[1]
\STATE \textbf{Input:} prompt $x$, LM $p_\theta$, exponent $\alpha$, particles $N$, ESS threshold $\kappa\in(0,1)$, max tokens $T_{\max}$, proposal $q_t$
\STATE Initialize $y_{1:0}^{(i)}=\emptyset$, weights $\tilde W_0^{(i)}=1$ for $i=1,\dots,N$
\STATE Initialize termination flags $\mathsf{done}^{(i)}\leftarrow \texttt{false}$ for all $i$
\FOR{$t=1$ to $T_{\max}$}
  \FOR{each particle $i=1,\dots,N$ \textbf{(parallel)}}
    \IF{$\mathsf{done}^{(i)}$}
      \STATE set $y_t^{(i)}\leftarrow \EOS$ and keep $\tilde W_t^{(i)}\leftarrow \tilde W_{t-1}^{(i)}$ \textbf{(stays terminated)}
    \ELSE
      \STATE sample $y_t^{(i)}\sim q_t(\cdot\mid x,y_{<t}^{(i)})$
      \STATE update $\tilde W_t^{(i)} \leftarrow \tilde W_{t-1}^{(i)}\cdot \dfrac{p_\theta(y_t^{(i)}\mid x,y_{<t}^{(i)})^\alpha}{q_t(y_t^{(i)}\mid x,y_{<t}^{(i)})}$
      \IF{$y_t^{(i)}=\EOS$}
        \STATE set $\mathsf{done}^{(i)}\leftarrow \texttt{true}$
      \ENDIF
    \ENDIF
  \ENDFOR
  \STATE normalize $W_t^{(i)}\propto \tilde W_t^{(i)}$; compute $\ESS_t$ via~\eqref{eq:ess}
  \IF{$\ESS_t<\kappa N$}
    \STATE \textbf{systematic resampling:} draw ancestors $A_{1:N}\leftarrow \mathrm{SysResample}(W_t^{(1:N)})$
    \STATE set $y_{1:t}^{(k)}\leftarrow y_{1:t}^{(A_k)}$, reorder KV cache by $A_{1:N}$ (Appendix~\ref{sec:systems})
    \STATE set $\mathsf{done}^{(k)}\leftarrow \mathsf{done}^{(A_k)}$ and reset $\tilde W_t^{(k)}\leftarrow 1$ for all $k$
  \ENDIF
\ENDFOR
\STATE \textbf{Output:} sample $I\sim \mathrm{Categorical}(W_{T_{\max}}^{(1:N)})$ and return $y^{(I)}$.
\end{algorithmic}
\end{algorithm}

\section{Background: Importance Sampling and Sequential Monte Carlo}
\label{app:bg-is-smc}

Suppose we wish to compute expectations under a target distribution $\pi(y)=\gamma(y)/Z$ where $\gamma(y)\ge 0$ is known but the normalizing constant $Z=\sum_y \gamma(y)$ is not.
Given samples $y^{(i)}\sim q(y)$ from a proposal distribution~$q$, importance sampling assigns each sample an unnormalized weight
\begin{equation}
w(y)\;=\;\frac{\gamma(y)}{q(y)},
\end{equation}
and approximates target expectations via
\begin{equation}
\E_{\pi}[f(Y)]
\;\approx\;
\frac{\sum_i w(y^{(i)})\,f(y^{(i)})}{\sum_i w(y^{(i)})}
\qquad\text{(self-normalized IS).}
\end{equation}
When the target is a distribution over long sequences, applying IS ``all at once'' yields extremely high-variance weights.
Sequential Monte Carlo can be viewed as applying IS \emph{incrementally} along the sequence.

SMC maintains $N$ weighted samples (particles) that evolve over time.
At each step, each particle proposes a new token, and its weight is multiplied by an \emph{incremental importance weight} that corrects the proposal toward the desired target.
When the weights become too uneven, SMC performs \emph{resampling}: particles with large weights are duplicated and particles with small weights are discarded, after which weights are reset.
A standard diagnostic for weight collapse is the \emph{effective sample size} (ESS), defined in Eq.~\eqref{eq:ess} of the main text.

\section{R\'enyi-Entropy Interpretation and $\alpha$-Ramping}
\label{app:renyi}

\subsection{R\'enyi-entropy interpretation of residual weight dispersion}

Define the prefix-dependent $\alpha$-power normalizer
\begin{equation}
Z_t(\alpha;\, y_{<t}) := \sum_v p_t(v)^\alpha.
\label{eq:Zt}
\end{equation}
Under $q_t^\star$, the incremental weight~\eqref{eq:omega} equals $\omega_t\equiv Z_t(\alpha;\, y_{<t})$, which depends on the prefix but not on the sampled token.

To interpret $Z_t$, recall the R\'enyi entropy of order~$\alpha$ for a discrete distribution~$p$:
\begin{equation}
H_\alpha(p) := \frac{1}{1-\alpha}\log \left(\sum_v p(v)^\alpha\right),
\qquad \alpha>0,\;\alpha\neq 1.
\end{equation}
Applying this to $p_t(\cdot)=p_\theta(\cdot\mid x,y_{<t})$ yields
\begin{equation}
\log Z_t(\alpha;\, y_{<t})
\;=\;
(1-\alpha)\,H_\alpha\!\bigl(p_\theta(\cdot\mid x,y_{<t})\bigr).
\label{eq:renyiZ}
\end{equation}
Since $\sum_v p_t(v)^\alpha \le 1$ for $\alpha\ge 1$, we have $\log Z_t(\alpha;\, y_{<t})\le 0$.

\paragraph{Path-wise weight accumulation (SIS view).}
Consider the underlying \emph{sequential importance sampling} (SIS) weights before any resampling resets.
For particle~$i$, the accumulated log-weight at time $T_{\max}$ under the locally optimal proposal is (restating Eq.~\ref{eq:pathweight} from the main text with the intermediate step):
\begin{equation}
\log \tilde W_{T_{\max}}^{(i)}
\;=\;
\sum_{t=1}^{T_{\max}} \log Z_t(\alpha;\, y_{<t}^{(i)})
\;=\;
(1-\alpha)\sum_{t=1}^{T_{\max}} H_\alpha\!\bigl(p_\theta(\cdot\mid x,y_{<t}^{(i)})\bigr),
\tag{\ref{eq:pathweight}}
\end{equation}
which is non-positive.
Particles traversing prefixes with \emph{higher} next-token uncertainty (larger R\'enyi entropy) accumulate \emph{lower} weights.
Conversely, particles on more ``confident'' prefix paths receive higher weight, consistent with the sharpening intent of the power distribution.

Equation~\eqref{eq:pathweight} makes the remaining source of degeneracy concrete:
even with the locally optimal proposal, weights diverge when different particles encounter prefixes whose next-token uncertainty differs substantially.
This motivates the $\alpha$-ramping schedule below.

\subsection{Exact $\alpha$-ramping}
\label{app:ramping}

To mitigate path-level weight divergence \emph{without changing the final target} $\pi_\alpha$, we introduce an $\alpha$-ramping schedule
\[
1=\alpha^{(0)}<\alpha^{(1)}<\cdots<\alpha^{(L)}=\alpha,
\]
and define intermediate targets $\gamma_t^{(\ell)}(y_{1:t}\mid x)\propto p_\theta(y_{1:t}\mid x)^{\alpha^{(\ell)}}$.
Within stage $\ell$, the correct incremental weight is obtained by replacing $\alpha$ with $\alpha^{(\ell)}$ in~\eqref{eq:inc}.
At stage boundaries (after a chosen token index $t$), we apply the standard SMC-samplers reweighting update~\citep{delmoral2006smc}:
\[
\log \tilde W \;\leftarrow\; \log \tilde W + (\alpha^{(\ell)}-\alpha^{(\ell-1)})\cdot \log p_\theta(y_{1:t}\mid x),
\]
where $\log p_\theta(y_{1:t}\mid x)=\sum_{s=1}^t \log p_\theta(y_s\mid x,y_{<s})$ is available from the autoregressive factors.
This transitions the target from $\gamma^{(\ell-1)}$ to $\gamma^{(\ell)}$ while preserving correctness of the final target.
Since $\prod_\ell p^{\alpha^{(\ell)}-\alpha^{(\ell-1)}}=p^{\alpha^{(L)}-\alpha^{(0)}}=p^{\alpha-1}$, the cumulative reweighting is identical to directly targeting the final exponent $\alpha$, preserving correctness.
Within stage~$\ell$, the locally optimal prefix-only proposal is $q_{t,\ell}^\star(\cdot)\propto p_t(\cdot)^{\alpha^{(\ell)}}$, i.e., temperature $\tau_\ell=1/\alpha^{(\ell)}$.
In our experiments, we use a simple linear schedule $\alpha^{(\ell)}=1+(\alpha-1)\cdot\ell/L$ over the first $T_{\mathrm{ramp}}$ tokens.

\section{Compute and Latency Cost Analysis: MH vs.\ Power-SMC}
\label{app:cost}

We formalize the latency advantage of Power-SMC over MH under an engine-independent cost model.
Both methods use KV caching, so the dominant cost is the number of incremental decode steps.

\subsection{Setup, cost model, and notation}

Let $p_\theta(\cdot\mid x)$ be an autoregressive LM and consider sampling from $\pi_\alpha(y_{1:T}\mid x)\propto p_\theta(y_{1:T}\mid x)^\alpha$.
We define one \emph{token-eval} as a single cached forward step producing next-token logits for one sequence, so a decode step at batch size~$b$ costs $b$ token-evals.
To translate token-evals into wall-clock time, let $s(b)\ge 1$ be the \emph{batch throughput multiplier} at batch size~$b$ relative to batch~$1$; a batch-$b$ step then takes time proportional to $b/s(b)$, capturing the sub-linear scaling of hardware utilization.
Throughout, $T$ is the number of generated tokens, $B$ the block length ($K:=T/B$ blocks, $B\mid T$), $M$ the number of MH moves per block, and $N$ the number of SMC particles.

\subsection{Cost of Power-SMC}

We begin with Power-SMC because its cost is straightforward and serves as the baseline for comparison.

\begin{lemma}[Power-SMC cost]
\label{lem:smc-cost}
Under KV caching, Power-SMC with $N$ particles and horizon $T$ performs $C_{\mathrm{SMC}} = N\cdot T$ token-evals, with wall-clock time proportional to $\mathrm{Time}_{\mathrm{SMC}} \propto T\cdot N/s(N)$.
\end{lemma}

\begin{proof}
Each of $T$ decode steps advances all $N$ particles (one batch-$N$ forward pass: $N$ token-evals), yielding $NT$ total.
Weight updates and resampling are $O(N)$ per step and do not change the leading term.
Wall-clock time follows from $T$ steps at cost $N/s(N)$ each.
\end{proof}

\subsection{Cost of MH power sampling}

We next analyze the MH construction of~\citet{karan2026reasoning}, where each move selects an edit point and regenerates the suffix autoregressively.
We consider two edit-index regimes to separate algorithmic structure from implementation choices.

\paragraph{General decomposition.}
Let $L_{k,m}$ be the regenerated suffix length in MH move $m\in\{1,\dots,M\}$ within block $k\in\{1,\dots,K\}$.
Block extension costs $B$ token-evals and each move costs $L_{k,m}$, so
\begin{equation}
C_{\mathrm{MH}} = \sum_{k=1}^K \Bigl(B + \sum_{m=1}^M L_{k,m}\Bigr) = T + \sum_{k=1}^K \sum_{m=1}^M L_{k,m},
\label{eq:mh-general}
\end{equation}
with expectation $\E[C_{\mathrm{MH}}] = T + M\sum_{k=1}^K \E[L_k]$, where $L_k$ is the suffix length in a representative move within block~$k$.

\begin{lemma}[Global-edit MH cost]
\label{lem:mh-global}
If each move in block $k$ edits uniformly over the full prefix of length $t_k=kB$, then $\E[L_k]\approx kB/2$ and
$\E[C_{\mathrm{MH}}] \approx T\!\bigl(1 + M(K{+}1)/4\bigr)$,
which is $\Theta(T^2/B)$ for fixed~$B$.
\end{lemma}

\begin{proof}
Uniform edit on $\{0,\dots,t_k-1\}$ gives $\E[L_k]=(t_k+1)/2\approx kB/2$.
Summing: $\E[C_{\mathrm{MH}}] \approx T + (MB/2)\cdot K(K{+}1)/2 = T + MBK(K{+}1)/4$; substituting $T=KB$ yields the result.
\end{proof}

\begin{lemma}[Last-block edit MH cost]
\label{lem:mh-lastblock}
If each move edits uniformly within the most recent block only, then $\E[L_k]\approx B/2$ for all~$k$ and
$\E[C_{\mathrm{MH}}] \approx T(1+M/2)$.
\end{lemma}

\begin{proof}
The offset to block end is uniform on $\{1,\dots,B\}$, so $\E[L_k]=(B{+}1)/2\approx B/2$.
Substituting: $\E[C_{\mathrm{MH}}] \approx T + MK(B/2)=T(1+M/2)$.
\end{proof}

\subsection{Compute and latency ratios}

Combining the results above directly yields the following comparisons.

\begin{corollary}[Compute ratio: global-edit MH vs.\ Power-SMC]
\label{cor:ratio-compute}
$\E[C_{\mathrm{MH}}]/C_{\mathrm{SMC}} \approx \bigl(1 + M(K{+}1)/4\bigr)/N.$
\end{corollary}

\begin{corollary}[Wall-clock ratio]
\label{cor:ratio-lat}
Assuming MH moves execute at batch~$1$ while Power-SMC uses batch~$N$,
\[
\mathrm{Time}_{\mathrm{MH}}\big/\mathrm{Time}_{\mathrm{SMC}}
\;\approx\;
\bigl(1+M(K{+}1)/4\bigr)\cdot s(N)/N
\quad\text{(global-edit)}.
\]
\end{corollary}

\begin{proof}
$\mathrm{Time}_{\mathrm{MH}}\propto \E[C_{\mathrm{MH}}]$ (batch~$1$) and $\mathrm{Time}_{\mathrm{SMC}}\propto T\cdot N/s(N)$ (Lemma~\ref{lem:smc-cost}).
\end{proof}

\begin{corollary}[MH overhead floor under last-block edit]
\label{cor:mh-floor}
The expected MH overhead relative to standard output decoding satisfies
$\rho_{\mathrm{MH}} \gtrsim 1+M/2$;
for $M=10$ this gives $\rho_{\mathrm{MH}}\gtrsim 6$ even under a perfect inference engine.
\end{corollary}

\noindent
As a concrete example, if $N=48$, $M=10$, and $K=16$, the global-edit compute factor is $1+M(K{+}1)/4=43.5$ and $\E[C_{\mathrm{MH}}]/C_{\mathrm{SMC}}\approx 0.91$.
Even here, Power-SMC can be wall-clock favorable when $s(48)$ is large, because its additional compute is batch-parallel rather than serial.

\subsection{Per-benchmark latency measurements}
\label{app:latency_full}

Table~\ref{tab:latency_full} reports the wall-clock latency of MH power sampling and Power-SMC across all 16 model--benchmark configurations.
These are the values that underlie the headline ranges quoted in the main text (Section~\ref{sec:experiments}).
All measurements use the same Hugging Face evaluation stack and hardware as the main results table.

\begin{table}[h]
\centering
\small
\setlength{\tabcolsep}{6pt}
\renewcommand{\arraystretch}{1.15}
\begin{tabular}{llcccc}
\toprule
\textbf{Model} & \textbf{Method} & \textbf{MATH500} & \textbf{GSM8K} & \textbf{GPQA} & \textbf{HumanEval} \\
\midrule
\multirow{2}{*}{\texttt{Qwen2.5-7B}}
& MH power sampling & 28.30$\times$ & 18.7$\times$ & 30.6$\times$ & 12.4$\times$ \\
& \cellcolor[HTML]{D5E8D4}\textbf{Power-SMC (ours)} & \cellcolor[HTML]{D5E8D4}\textbf{1.61$\times$} & \cellcolor[HTML]{D5E8D4}\textbf{1.95$\times$} & \cellcolor[HTML]{D5E8D4}\textbf{0.76$\times$} & \cellcolor[HTML]{D5E8D4}\textbf{1.85$\times$} \\
\midrule
\multirow{2}{*}{\texttt{Qwen2.5-Math-7B}}
& MH power sampling & 18.20$\times$ & 14.3$\times$ & 37.8$\times$ & 15.5$\times$ \\
& \cellcolor[HTML]{D5E8D4}\textbf{Power-SMC (ours)} & \cellcolor[HTML]{D5E8D4}\textbf{1.46$\times$} & \cellcolor[HTML]{D5E8D4}\textbf{1.52$\times$} & \cellcolor[HTML]{D5E8D4}\textbf{0.97$\times$} & \cellcolor[HTML]{D5E8D4}\textbf{1.61$\times$} \\
\midrule
\multirow{2}{*}{\texttt{Phi-3.5-mini-instruct}}
& MH power sampling & 16.12$\times$ & 12.4$\times$ & 19.3$\times$ & 14.1$\times$ \\
& \cellcolor[HTML]{D5E8D4}\textbf{Power-SMC (ours)} & \cellcolor[HTML]{D5E8D4}\textbf{3.14$\times$} & \cellcolor[HTML]{D5E8D4}\textbf{2.37$\times$} & \cellcolor[HTML]{D5E8D4}\textbf{2.00$\times$} & \cellcolor[HTML]{D5E8D4}\textbf{2.80$\times$} \\
\midrule
\multirow{2}{*}{\texttt{Qwen3-4B}}
& MH power sampling & 19.34$\times$ & 21.5$\times$ & 43.0$\times$ & 10.3$\times$ \\
& \cellcolor[HTML]{D5E8D4}\textbf{Power-SMC (ours)} & \cellcolor[HTML]{D5E8D4}\textbf{1.55$\times$} & \cellcolor[HTML]{D5E8D4}\textbf{1.01$\times$} & \cellcolor[HTML]{D5E8D4}\textbf{1.47$\times$} & \cellcolor[HTML]{D5E8D4}\textbf{2.10$\times$} \\
\bottomrule
\end{tabular}
\caption{\textbf{Per-benchmark wall-clock latency} (relative to standard output decoding for each model and benchmark).
Across all 16 configurations, MH ranges from $10.3\times$ to $43.0\times$ overhead; Power-SMC ranges from $0.76\times$ to $3.14\times$.
The corresponding speedup of Power-SMC over MH ranges from approximately $5\times$ to $40\times$.}
\label{tab:latency_full}
\end{table}

\section{Memory Overhead}
\label{app:memory}

Maintaining $N=64$ particles requires storing $N$ separate KV caches.
We apply two optimizations to reduce the memory cost.

The first is \emph{shared prompt caching}.
The prompt is processed once at batch size~1, and the resulting cache is replicated to all $N$ particles.
This avoids redundant computation during the prompt phase.

The second is \emph{copy-on-write (CoW) deduplication} after resampling.
Systematic resampling typically produces only $U\ll N$ unique ancestors.
For example, with $N{=}64$ we observe $U{\approx}14$ on average.
CoW retains only $U$ physical caches and defers expansion until particles sample different tokens and actually diverge.
This reduces transient peak memory from $2N$ to $N{+}U$ cache equivalents.
Table~\ref{tab:memory} reports peak GPU memory.

\begin{table}[h]
\centering
\small
\setlength{\tabcolsep}{6pt}
\renewcommand{\arraystretch}{1.15}
\begin{tabular}{lccc}
\toprule
& \textbf{Standard output} & \textbf{MH} & \textbf{Power-SMC ($N{=}64$)} \\
\midrule
GSM8K ($T_{\max}{=}2048$) & 14.59\,GB & 15.01\,GB & 16.20\,GB \\
MATH500 ($T_{\max}{=}3072$) & 15.45\,GB & 16.54\,GB & 18.38\,GB \\
\bottomrule
\end{tabular}
\caption{\textbf{Peak GPU memory} on \texttt{Qwen2.5-7B} (float16 weights).
Measured via \texttt{torch.cuda.max\_memory\_allocated}, averaged across prompts.
Power-SMC with 64 particles adds at most 2.9\,GB over standard output.
This modest overhead is due to copy-on-write cache deduplication after resampling.}
\label{tab:memory}
\end{table}

\section{Ablation Studies}
\label{app:ablations}

\subsection{Effect of the $\alpha$-ramping schedule}
\label{app:ablation_ramp}

The analysis in Eq.~\eqref{eq:pathweight} identifies the source of weight collapse under the optimal temperature: path-level accumulation of differences in next-token uncertainty.
The $\alpha$-ramping schedule addresses this.
By gradually increasing the exponent from $\alpha{=}1$ to its final value, it allows particles to explore diverse reasoning paths before aggressive sharpening begins.
We ablate the ramp length $T_{\mathrm{ramp}}$ on MATH500 using \texttt{Qwen2.5-Math-7B} ($\alpha{=}4$, $N{=}64$, block size~64, $\kappa{=}0.5$).
Table~\ref{tab:ablation_ramp} reports four diagnostics.
These are: pass@1 accuracy, mean number of resampling events, ESS at the first block boundary (out of $N{=}64$), and mean number of unique ancestors $U$ per resampling event.

\begin{table}[h]
\centering
\small
\setlength{\tabcolsep}{5pt}
\renewcommand{\arraystretch}{0.95}
\begin{tabular}{rcccc}
\toprule
$T_{\mathrm{ramp}}$ & \textbf{Acc.\,(\%)} & \textbf{Resample count} & \textbf{ESS @ 1st block} & \textbf{Unique $U$} \\
\midrule
0 (no ramp) & 74.0 & 8.0 & 14.5 & 10 \\
50 & 75.0 & 7.0 & 19.0 & 12 \\
\textbf{100} & \textbf{76.2} & \textbf{5.5} & \textbf{27.5} & \textbf{14} \\
200 & 76.0 & 4.5 & 35.5 & 18 \\
400 & 75.4 & 2.5 & 48.0 & 21 \\
\bottomrule
\end{tabular}
\caption{\textbf{Ablation: $\alpha$-ramping schedule} on MATH500 (\texttt{Qwen2.5-Math-7B}, $\alpha{=}4$, $N{=}64$).
$T_{\mathrm{ramp}}{=}0$ applies the full exponent from the first token; larger values ramp $\alpha$ linearly from 1 to 4 over the specified number of tokens.
ESS is measured at the first block boundary ($t{=}64$); unique ancestors $U$ is averaged over all resampling events.}
\label{tab:ablation_ramp}
\end{table}

The results confirm the R\'enyi-entropy analysis.
Without ramping ($T_{\mathrm{ramp}}{=}0$), the full exponent $\alpha{=}4$ applies from the first token.
Path-level weight divergence accumulates rapidly: by the first block boundary, ESS has dropped to 14.5 out of 64, triggering aggressive resampling that leaves only $U{\approx}10$ unique ancestors.
This early particle collapse permanently reduces diversity, and accuracy falls to 74.0\%, 2.2 points below the optimum.
The connection to Eq.~\eqref{eq:pathweight} is direct.
At full $\alpha$, the accumulated log-weight $(1{-}\alpha)\sum_t H_\alpha$ diverges quickly across particles that encounter prefixes with different next-token entropy.
A handful of ``confident'' particles dominate, and the rest are effectively wasted.

As $T_{\mathrm{ramp}}$ increases, the early tokens accumulate weight at a reduced exponent, damping the path-level divergence.
At $T_{\mathrm{ramp}}{=}100$, the exponent has reached only $\alpha{\approx}2.9$ by the first block boundary.
ESS remains at 27.5, nearly twice the no-ramp value, and resampling preserves $U{\approx}14$ unique ancestors, maintaining sufficient diversity for the candidate sequences to explore different solution strategies.
Accuracy peaks at 76.2\%.

Extending the ramp further ($T_{\mathrm{ramp}}{=}200$--$400$) continues to improve early ESS (35.5 and 48.0, respectively) and yields even more unique ancestors per resample.
However, accuracy slightly decreases.
The reason is that an excessively slow ramp delays the onset of meaningful sharpening.
At $T_{\mathrm{ramp}}{=}400$, the exponent is only $\alpha{\approx}1.5$ by the first block boundary.
Particles spend many tokens under near-baseline sampling before the weights have enough discriminative power to separate good from bad reasoning paths.
This allows particles to commit to suboptimal strategies before corrective pressure takes effect.

The ablation reveals a trade-off.
Too-fast ramping causes weight collapse: a few particles dominate and diversity is lost.
Too-slow ramping causes insufficient early sharpening: particles commit to weak strategies before the weights can redirect them.
The optimum at $T_{\mathrm{ramp}}{\approx}100$ balances these forces.
Results are robust across the $T_{\mathrm{ramp}}\in[50,200]$ range.

\subsection{Effect of the ESS resampling threshold}
\label{app:ablation_ess}

The ESS threshold $\kappa$ controls how aggressively Power-SMC resamples.
Resampling triggers when $\ESS_t < \kappa N$ (Eq.~\ref{eq:ess}).
Setting $\kappa{=}0$ disables resampling entirely, reducing Power-SMC to pure sequential importance sampling (SIS).
Setting $\kappa{=}1$ forces resampling at every block boundary regardless of weight quality.
We sweep $\kappa\in\{0, 0.25, 0.5, 0.75, 1.0\}$ on MATH500 using \texttt{Qwen2.5-Math-7B} ($\alpha{=}4$, $N{=}64$, block size~64, $T_{\mathrm{ramp}}{=}100$).
Table~\ref{tab:ablation_ess} reports pass@1 accuracy, mean resample count, mean ESS at the moment resampling triggers, and mean unique ancestors~$U$.

\begin{table}[h]
\centering
\small
\setlength{\tabcolsep}{5pt}
\renewcommand{\arraystretch}{1.15}
\begin{tabular}{rcccc}
\toprule
$\kappa$ & \textbf{Acc.\,(\%)} & \textbf{Resample count} & \textbf{ESS @ trigger} & \textbf{Unique $U$} \\
\midrule
0.0 (no resample) & 72.6 & 0 & --- & --- \\
0.25 & 74.6 & 3.5 & 10.7 & 8 \\
\textbf{0.5} & \textbf{76.2} & \textbf{7.0} & \textbf{18.3} & \textbf{13} \\
0.75 & 75.6 & 12.0 & 24.9 & 18 \\
1.0 (always resample) & 75.2 & 19.5 & 45.0 & 22 \\
\bottomrule
\end{tabular}
\caption{\textbf{Ablation: ESS resampling threshold $\kappa$} on MATH500 (\texttt{Qwen2.5-Math-7B}, $\alpha{=}4$, $N{=}64$).
$\kappa{=}0$ disables resampling (pure SIS); $\kappa{=}1$ resamples at every block boundary.
``ESS @ trigger'' reports the mean ESS when resampling fires.}
\label{tab:ablation_ess}
\end{table}

The most informative result is the $\kappa{=}0$ row.
Without resampling, all $N{=}64$ particles evolve independently.
A single sample is then drawn from the final weights.
As Eq.~\eqref{eq:pathweight} predicts, weight divergence causes a small number of particles to accumulate the vast majority of the mass.
The final selection effectively reduces to picking a single sequence.
The 3.6-point gap between $\kappa{=}0$ (72.6\%) and $\kappa{=}0.5$ (76.2\%) confirms that resampling is essential.
It is the mechanism that redirects computation from unpromising to promising reasoning paths.

At the other extreme, $\kappa{=}1.0$ forces resampling at every block boundary, regardless of weight quality.
This maintains very uniform weights (ESS${\approx}$45 at trigger) and high ancestor counts ($U{\approx}22$).
However, accuracy drops by 1.0 point relative to $\kappa{=}0.5$.
The cause is \emph{particle impoverishment}.
With resampling every 64 tokens, particles exploring unconventional but potentially correct reasoning strategies are killed before they can prove their value.
Only particles that appear promising within a single block survive.
This introduces a short-horizon bias.

The intermediate thresholds reveal an asymmetric trade-off.
Under-resampling ($\kappa{=}0$: $-$3.6 points) is far more damaging than over-resampling ($\kappa{=}1.0$: $-$1.0 point).
Practitioners should err on the side of more frequent resampling.
The optimal $\kappa{=}0.5$ catches weight imbalance at a moderate ESS (${\approx}18$).
This preserves $U{\approx}13$ distinct candidate sequences, enough diversity to explore different solution strategies while concentrating computation on the most promising ones.

\section{Systems: Cache-Safe Resampling for Transformer Decoding}
\label{sec:systems}

When a high-weight particle is duplicated during resampling, its model state must be copied as well.
For autoregressive Transformers, the dominant state is the \emph{KV cache}: the stored attention keys and values for each particle's prefix.
KV cache layouts differ across model architectures and library versions.
We implement cache reordering with a three-tier strategy.
First, we use model-provided cache-reordering hooks when available (e.g., \texttt{\_reorder\_cache}).
Second, we fall back to cache-object reorder methods exposed by the runtime.
Third, if neither is available, we apply a recursive tensor reindexer that traverses nested containers and reindexes along the batch dimension without assuming a specific internal structure.
This makes resampling correct and efficient across common Hugging Face backends.

\section{EOS and Variable-Length Decoding}
\label{app:eos}

We treat $\EOS$ as an ordinary token in $\vocab\cup\{\EOS\}$.
Once a particle emits $\EOS$, it transitions to an absorbing state: subsequent steps apply a no-op transition with incremental weight~$1$.
In implementation, this is achieved by masking proposals to force $\EOS$ for terminated particles and skipping cache updates for those particles.

\section{Resampling Choices}
\label{app:resampling}

Our implementation uses \emph{systematic resampling} (Algorithm~\ref{alg:smc}), which is unbiased and typically lower-variance than multinomial resampling~\citep{doucet2011tutorial}.
More broadly, any standard unbiased resampling scheme (multinomial, stratified, residual, systematic) preserves the target distribution; these choices primarily affect variance and particle diversity.

\section{Setup for GRPO and Best-of-$N$ Comparison}
\label{app:grpo_bon}

This appendix provides setup details for the comparison reported in Table~\ref{tab:grpo_bon} of the main text.
For GRPO, we use the post-trained models released by~\citet{karan2026reasoning}.
For Best-of-$N$, we follow the protocol of~\citet{ji2026scalable}.
We draw $N_{\mathrm{BoN}}=64$ independent completions from the un-finetuned base model, matching the $N=64$ candidate sequences used by Power-SMC.
We then select the completion with the highest log-likelihood, also computed under the base model.

\paragraph{Note on GPQA parsing.}
Throughout our experiments, we score GPQA responses by extracting an exact multiple-choice letter (e.g., \texttt{A}, \texttt{B}, \texttt{C}, or \texttt{D}) from the model's output.
Numeric values, free-form explanations, or other non-letter answers are discarded as incorrect, even if they describe the correct option.
This convention is applied uniformly across all methods and models, including GRPO, MH power sampling, Best-of-$N$, and standard output decoding.
Some of our GPQA numbers for shared models match those reported by~\citet{karan2026reasoning} and~\citet{ji2026scalable}, while others differ.
Where readers observe a discrepancy, it is due to this stricter parsing convention rather than to a difference in method or sampling configuration.
The within-table comparisons in this paper remain valid because all methods are evaluated under the same parser.

\section{Limitations}
\label{app:limitations}

Power-SMC inherits the structural limitations of any training-free, inference-time sampler.
First, its accuracy is bounded by what the base model already supports: if the pretrained distribution assigns negligible probability to correct solutions, no amount of sharpening can recover them.
Methods that modify model parameters, such as RL post-training, can in principle move probability mass onto sequences that the base model would not generate at all.
Second, Power-SMC trades latency for memory.
Decoding $N$ candidate sequences in parallel requires roughly $N$ times the KV-cache footprint of single-sequence decoding.
The copy-on-write strategy described in Appendix~\ref{app:memory} keeps this overhead modest in practice, but very long contexts or very large $N$ may exceed the memory of a single GPU.
Finally, our theoretical guarantees concern per-step importance weight variance and assume the proposal class that depends only on the tokens generated so far.
Methods that use future-token information (for example, rollout-based estimates~\citep{ji2026scalable}) lie outside this class, and our optimality result does not apply to them.

\end{document}